\documentclass[journal]{IEEEtran}

\usepackage{cite}
\usepackage{amsmath}
\usepackage{amsfonts}
\usepackage{url}
\usepackage{graphicx}
\usepackage{subfigure}
\usepackage{caption}
 \usepackage{epstopdf}
\usepackage{multirow}
\usepackage{color,xcolor}
\usepackage{cite}
\usepackage{CJK}
\usepackage{comment}
\usepackage{amsthm}
\newtheorem{theorem}{Theorem}

\usepackage{verbatim}
\usepackage{algorithm}
\usepackage{algpseudocode}

\usepackage{bbding}
\usepackage{amssymb}
\usepackage{subcaption}

\usepackage{tabularx}

\def\Dset{\mathcal{D} }
\def\Tset{\mathcal{T} }
\def\Pset{\mathcal{P} }
\def\Qset{\mathcal{Q} }
\def\Lset{\mathcal{L} }

\def\x{\mathbf{x}}

\def\y{\mathbf{y}}

\def\z{\mathbf{z}}
\def\w{\mathbf{w}}

\begin{document}

\title{From Images to Point Clouds: An Efficient Solution for Cross-media Blind Quality Assessment without Annotated Training}

\author{Yipeng Liu, Qi Yang, Yujie Zhang, Yiling Xu, Le Yang, Zhu Li\thanks{This paper is supported in part by National Natural Science Foundation of China (62371290), National Key R\&D Program of China (2024YFB2907204), the Fundamental Research Funds for the Central Universities of China, and STCSM under Grant (22DZ2229005). The corresponding author is Yiling Xu(e-mail: yl.xu@sjtu.edu.cn).}
\thanks{Y. Liu, Y. Zhang and Y. Xu are from Cooperative Medianet Innovation Center, Shanghai Jiaotong University, Shanghai, 200240, China, (e-mail: liuyipeng@sjtu.edu.cn, yujie19981026@sjtu.edu.cn, yl.xu@sjtu.edu.cn)}
\thanks{Q. Yang and Z. Li are from University of Missouri-Kansas City, Missouri, US, (e-mail: qiyang@umkc.edu, lizhu@umkc.edu)}
\thanks{L. Yang is from the Department of electrical and computer engineering, University of Canterbury, Christchurch 8041, New Zealand, (e-mail: le.yang@canterbury.ac.nz)}
\thanks{Corresponding author: Y. Xu}
}

\maketitle

\begin{abstract}
We present a novel quality assessment method which can predict the perceptual quality of point clouds from new scenes without available annotations by leveraging the rich prior knowledge in images, called the Distribution-Weighted Image-Transferred Point Cloud Quality Assessment (DWIT-PCQA). Recognizing the human visual system (HVS) as the decision-maker in quality assessment regardless of media types, we can emulate the evaluation criteria for human perception via neural networks and further transfer the capability of quality prediction from images to point clouds by leveraging the prior knowledge in the images. Specifically, domain adaptation (DA) can be leveraged to bridge the images and point clouds by aligning feature distributions of the two media in the same feature space. However, the different manifestations of distortions in images and point clouds make feature alignment a difficult task. To reduce the alignment difficulty and consider the different distortion distribution during alignment, we have derived formulas to decompose the optimization objective of the conventional DA into two suboptimization functions with distortion as a transition. Specifically, through network implementation, we propose the distortion-guided biased feature alignment which integrates existing/estimated distortion distribution into the adversarial DA framework, emphasizing common distortion patterns during feature alignment. Besides, we propose the quality-aware feature disentanglement to mitigate the destruction of the mapping from features to quality during alignment with biased distortions. Experimental results demonstrate that our proposed method exhibits reliable performance compared to general blind PCQA methods without needing point cloud annotations.

\end{abstract}

% The existing learning-based NR quality assessment methods suffer from insufficient subjective annotations.                      


\IEEEpeerreviewmaketitle

\section{Introduction}\label{sec:introduction}

Point clouds have demonstrated remarkable performance in various applications, including augmented reality \cite{lim2020Augmented}, automatic driving \cite{ChenLFGW:20}, and industrial robots \cite{Rusu2011PCL}. Ensuring the accuracy of point cloud quality assessment (PCQA) is of paramount importance for delivering high-quality service for various human vision tasks. PCQA can be divided into subjective and objective methods. Although subjective methods can provide reliable quality prediction, they can be expensive in terms of time, cost, and testing conditions \cite{Mantiuk2012QAsurvey}. Therefore, the objective PCQA methods have become a hotspot in recent research, which is the main focus of this work.

\begin{figure}[t]
\setlength{\abovecaptionskip}{0.cm}
\setlength{\belowcaptionskip}{-0.cm}
	\centering
		\includegraphics[width=1\linewidth]{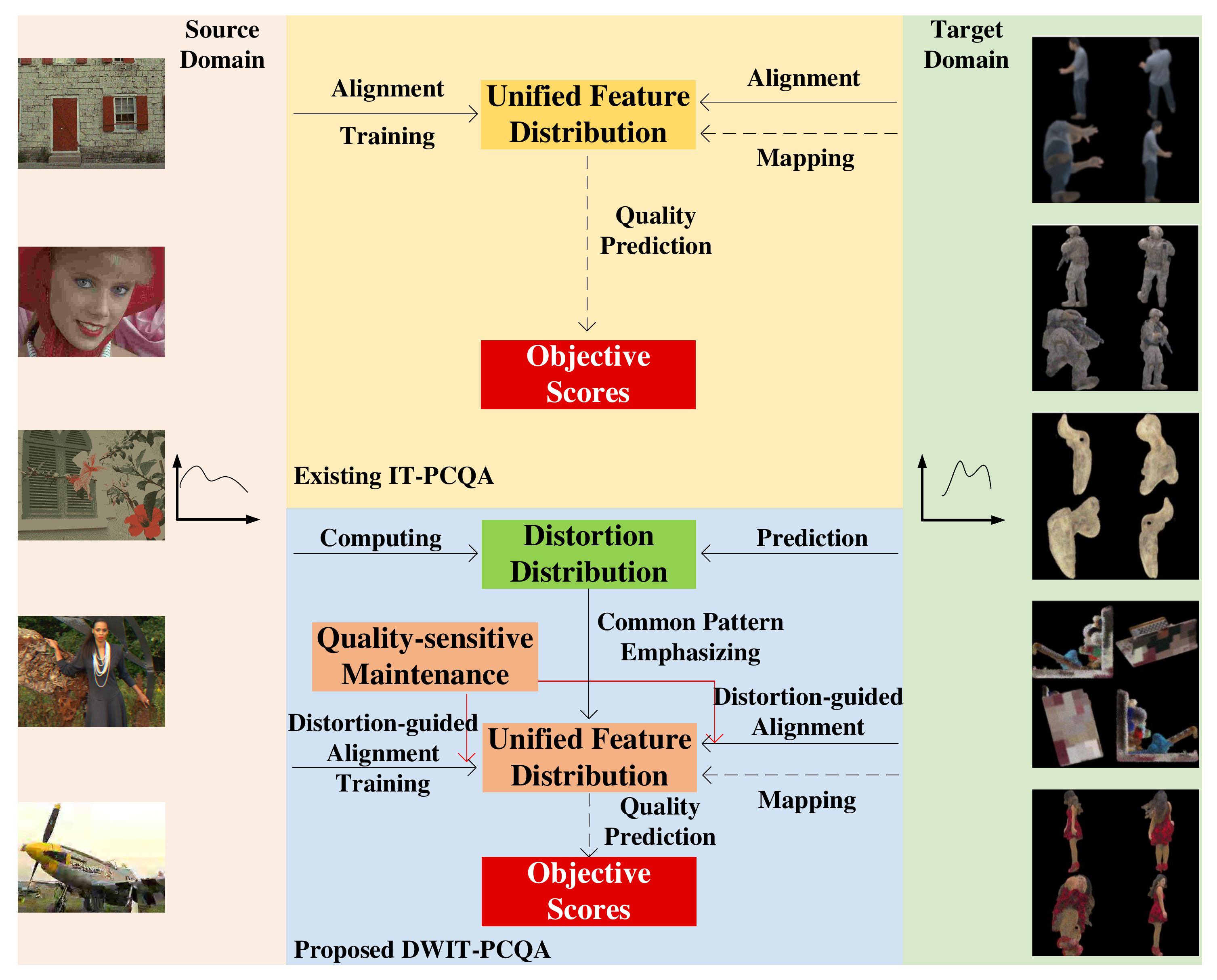}
	\caption{Comparison with existing ideas. The existing IT-PCQA directly aligns and maps the features into the unified distribution by domain adaptation (DA). The proposed DWIT-PCQA considers differential distortion distributions between the source domain and the target domain, and emphasizes common distortion patterns. Besides, contrastive learning is employed to promote the quality-sensitive representation during distortion-guided alignment. The dashed line represents the testing flow.}
	\label{fig:alignment}
\end{figure}

Objective methods can be categorized as full-reference (FR), reduced-reference (RR), and no-reference (NR) methods \cite{liu2024D3PCQA,liu2024VCIP}. FR and RR metrics require complete or partial reference information that is not available in many practical applications \cite{liu2022VCIP}, such as point clouds rendered after transmission \cite{LQrate} and captured in the wild. Therefore, a more urgent requirement for PCQA lies in the development of NR methods. An NR metric is usually based on the natural scene statistic (NSS) \cite{scstmm} or the deep neural network (DNN)~\cite{dnntip}. Both methods need to analyze the characteristics of the point cloud based on a large number of labeled samples, i.e., diverse distorted samples with mean opinion score (MOS). However, compared to well-studied image quality assessment (IQA), current PCQA databases are usually limited in scale and scene due to the difficulty of point cloud production and complex subjective experiments. For example, current widely-used PCQA datasets like PointXR \cite{Alexiou2020PointXR}, IRPC \cite{Javaheri2019IRPC}, SJTU-PCQA \cite{yang2020predicting}, and WPC \cite{Su2019WPC} only contain hundreds of labeled samples distorted from several handcrafted point clouds.

A solution to the above problem is to inherit prior knowledge from other mature research \cite{Yang2022ITPCQA}, e.g., IQA \cite{Li2024DA}, to facilitate the PCQA study. The feasibility comes from two key factors. First,  IQA databases, such as LIVE~\cite{LIVE}, CSIQ~\cite{Sheikh2006CSIQ}, TID2013~\cite{tid2013}, KonIQ \cite{Hosu2020konIQ} and KADID \cite{Lin2019KIAID10k}, which contain tens of thousands of samples, offer extensive and accurate subjective ratings and various capture scenes. Second, previous research has established strong connections between 2D and 3D perception, evident in tasks like reconstruction (e.g., from image to 3D object~\cite{2D3Dre}) and tracking ~\cite{2D3Dtracking}. Since the human visual system (HVS) serves as the universal evaluator, the perceptual characteristics observed in IQA likely exhibit homogeneity with those of PCQA. Considering the potential relationship between images and point clouds, it is promising to employ prior knowledge in IQA to guide the PCQA task.

IT-PCQA is the first and only existing attempt \cite{Yang2022ITPCQA} to solve NR-PCQA utilizing prior knowledge from IQA by domain adaptation (DA) \cite{NIPS2016_45fbc6d3, pmlr-v37-ganin15,Tzeng_2017_CVPR}. It uses DA to map the features of images and point clouds to the unified distribution by confusing a discriminator to make it unable to distinguish the features from the images or point clouds, and then the labels of images are used to train a quality regression, as shown in Fig. \ref{fig:alignment}. Intuitively, it is a prior knowledge transfer from synthesized distorted images to synthesized distorted point clouds (abbreviated as \textbf{image-to-point cloud transfer} in the following text). We inherit this transfer research, as existing PCQA research has focused mainly on synthesized distortions. However, the large domain gap makes the performance of such an attempt far from satisfactory. Although image distortions and partial distortions in point clouds share certain similarities in color/attribute~\cite{LIVE, tid2013}, their own unique distortions result in different distortion distributions, diminishing transfer performance after alignment. Fig. \ref{fig:tsne} shows the low-dimensional t-SNE plot of different distortion distributions extracted from images (TID2013 \cite{tid2013}) and point clouds (SJTU-PCQA \cite{yang2020predicting}), indicating that direct alignment is not the most appropriate choice especially due to the existence of some outlier distortions. Further, Fig. \ref{fig:deltaplcc} shows the impact of different distortions in source domain (TID2013) on the performance of IT-PCQA in target domain (SJTU-PCQA) by removing distortions with the median performance one by one. Despite the mutual influence between the samples, Fig. \ref{fig:deltaplcc} demonstrates that the presence of some distortions in the source domain suppresses the DA performance in the target domain. 

\begin{figure}[t]
\setlength{\abovecaptionskip}{0.cm}
\setlength{\belowcaptionskip}{-0.cm}
	\centering
		\includegraphics[width=1\linewidth]{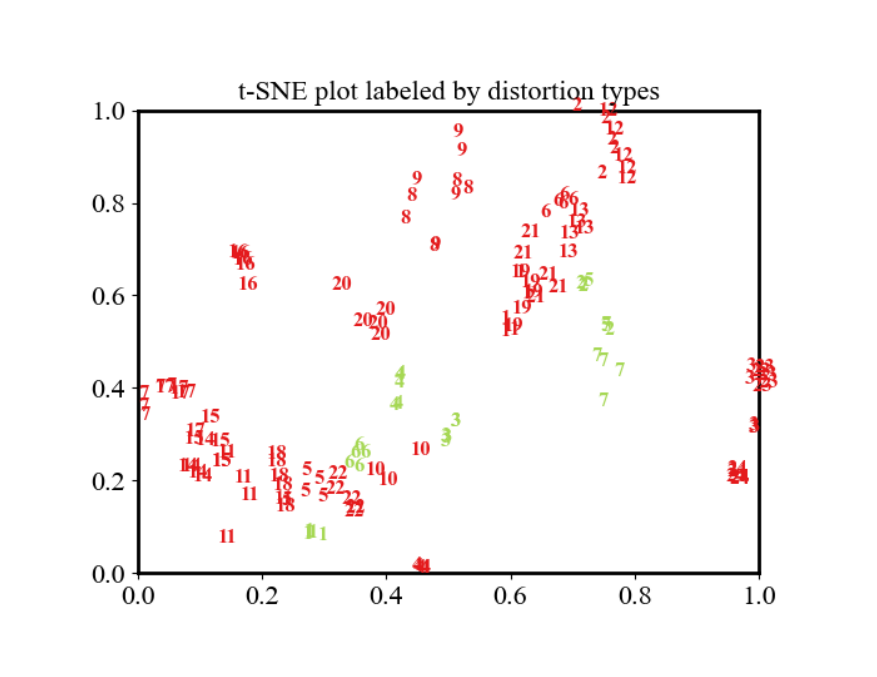}
	\caption{Different distortion distributions of images and point clouds. Features of images and projections are extracted from the pre-trained backbone of HyperIQA \cite{Su2020Hyper}. \textbf{Red} represents the feature distribution of images (TID2013 \cite{tid2013}), while \textbf{green} represents the feature distribution of point cloud projections (SJTU-PCQA \cite{yang2020predicting}). The numbers in the figure represent the distortion types numbering in their original papers.}
	\label{fig:tsne}
\end{figure}

\begin{figure}[pt]
\setlength{\abovecaptionskip}{0.cm}
\setlength{\belowcaptionskip}{-0.cm}
	\centering
		\includegraphics[width=1\linewidth]{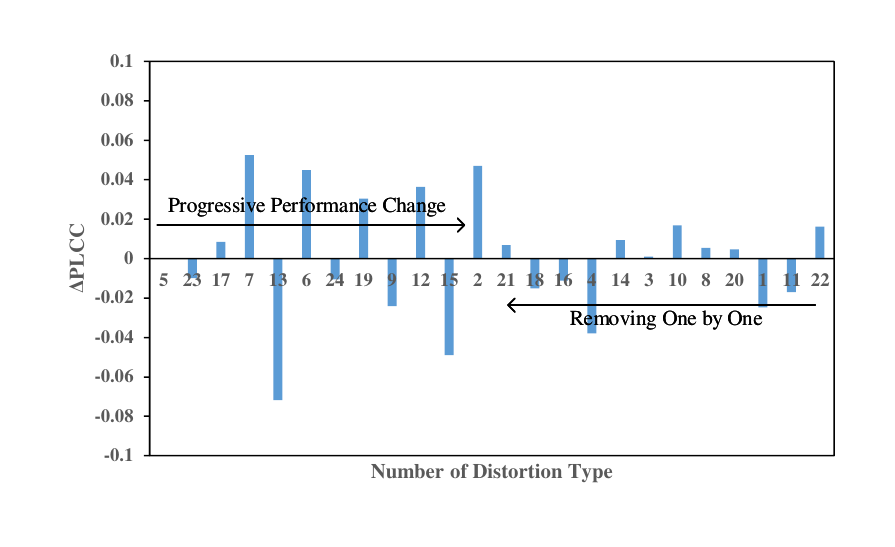}
	\caption{Performance change with each removal of distortion type in source domain of IT-PCQA (TID2013 as the source domain, and SJTU-PCQA as the target domain). The distortion type number is the same as in the original paper \cite{tid2013}, shown in Table \ref{tab:tiddistortions}. Positive $\Delta PLCC$ represents positive performance improvement after removing this distortion, indicating that the removed distortion has a potential negative impact on DA.}
	\label{fig:deltaplcc}
\end{figure}

\begin{table}[htbp]
  \centering
  \caption{Distortion types in the TID2013 dataset \cite{tid2013} and their numbering.}
    \begin{tabularx}{\linewidth}{cXcX}
    \hline
    \multicolumn{2}{c}{Distortion Types} & \multicolumn{2}{c}{Distortion Types} \\
    \hline
    1     & Additive Gaussian noise & 13    & JPEG2000 transmission errors \\
    2     & Additive noise  & 14    & Non eccentricity pattern noise \\
    3     & Spatially correlated noise & 15    & Local block-wise distortions \\
    4     & Masked noise & 16    & Mean shift (intensity shift) \\
    5     & High frequency noise & 17    & Contrast change \\
    6     & Impulse noise & 18    & Change of color saturation \\
    7     & Quantization noise & 19    & Multiplicative Gaussian noise \\
    8     & Gaussian blur & 20    & Comfort noise \\
    9     & Image denoising & 21    & Lossy compression of noisy images \\
    10    & JPEG compression & 22    & Image color quantization with dither \\
    11    & JPEG2000 compression & 23    & Chromatic aberrations \\
    12    & JPEG transmission errors & 24    & Sparse sampling and reconstruction \\
    \hline
    \end{tabularx}%
  \label{tab:tiddistortions}%
\end{table}%

% \begin{figure}[pt]
% \setlength{\abovecaptionskip}{0.cm}
% \setlength{\belowcaptionskip}{-0.cm}
% 	\centering
% 		\includegraphics[width=1\linewidth]{toy_example_DA_1.pdf}
% 	\caption{Motivation of domain adaption in PCQA. The source domain (images) and target domain (point clouds) have different domain distribution. Refer to the evaluation criterion specified by HVS, we use unsupervised adversarial domain adaptation (UADA)  to map the domain distribution and regress quality score. }
% 	\label{fig:uada}
% \end{figure}

% Unlike information transfer between different types of images where the main challenge is the content shift, such as image to screen content images (e.g., natural scenes to computer text, graph and etc.) \cite{ChenDA}, 

% Considering Considering the quality assessment task is essentially solving for the posterior probability of quality $\Dset(\y|\z)$,

The suppression comes from two aspects: difficult feature alignment and failed quality-sensitive maintenance. Specifically, given extracted features $\z$ and quality $\y$, different distortion manifestations lead to disparate low-level feature distributions $\Dset(\z)$ \cite{liu2019lowlevel}.  The bias in $\Dset(\z)$ caused by mismatched marginal distortion distortions $\Dset(\y_d)$ further results in the disparity of feature likelihood functions with respect to perceptual quality between the source domain $\Dset_S(\z|\y)$ and the target domain $\Dset_T(\z|\y)$. The likelihood conflict makes the DA network for feature alignment, which aims at $\min||\mathcal{D}_S(\z \mid \y) - \mathcal{D}_T(\z \mid \y)||$), difficult to converge and find a common quality mapping. Besides, due to the existence of different feature distributions between the source domain $\Dset_S(\z)$ and the target domain $\Dset_T(\z)$, direct cross-domain alignment may destroy the feature-to-quality mapping $\z_T \rightarrow \y$ in the target domain under $\z_S \rightarrow \y$ is well fitted. In other words, the alignment may suppress the quality-sensitive representation of the aligned features in the target domain. All these issues require us to emphasize common distortion patterns during cross-domain alignment.

% The difficulty of network convergence after feature confusion based on traditional DA lies in the mismatched marginal distortion distribution and the suppression of quality-sensitive feature representation. Specifically, , thereby necessitating the emphasis on common distortion patterns during feature alignment. Meanwhile, the perceptual quality is related to the image/point cloud itself, which is not equivalent to the distortion. Directly aligning the feature distribution based on distortion information may suppress the quality-sensitive representation, resulting in the failure of the feature-to-quality mapping in the target domain.

% Based on the theory of NSS and frequency analysis \cite{liu2019lowlevel}, different distortion distributions lead to different low-level feature distributions between two media. Since the perceptual quality is highly related to the distortions, directly aligning the different low-level feature distributions between two media destroys the likelihood functions of features with respect to perceptual quality. In this case, the learning of cross-domain commonality may suppress the quality-sensitive representation, resulting in the failure of the feature-to-quality mapping in the target domain.

% , and distinguishing content will further enhance domain discrepancy and suppress the learning of cross-domain commonality.  

% we prove by formula derivation that a distribution weight based on the marginal distortion distributions needs to be incorporated into feature distribution alignment to promote learning of cross-domain commonality.

In this paper, we propose a novel domain-transfer quality assessment framework called distribution-weighted image-transferred point cloud quality assessment (DWIT-PCQA) to solve the above problems, as shown in Fig. \ref{fig:alignment}. Specifically, to reduce alignment difficulty and emphasize common distortion patterns during cross-domain alignment, we decompose the optimization objective of the conventional DA into two suboptimization functions. The first suboptimization function constructs a \textbf{distortion-based conditional discriminator} to perform distortion-guided importance-weighted alignment in feature distributions, which emphasizes the cross-domain common distortion patterns, $\Dset_S(\y_d) \cap \Dset_T(\y_d)$, to align the feature likelihood functions with respect to common distortions $\Dset(\z|\y_d)$. The second suboptimization function constructs a \textbf{quality-aware feature disentanglement} to maintain the quality-sensitive components during alignment with biased distortion distributions. Considering that the perceptual quality is not equivalent to the distortion, only aligning the feature distribution based on distortion may suppress the feature-to-quality mapping in the target domain. The proposed feature disentanglement harmonizes mismatched feature likelihood functions with respect to distortion $\Dset(\z|\y_d)$ and quality $\Dset(\z|\y)$, by maintaining quality-sensitive components related to common distortions through contrastive learning. Finally, the point cloud quality can be predicted by using image annotations to make the aligned re-weighting features quality-aware while the quality scores of point clouds are not needed for training.

The contributions of this paper are summarized as follows:

\begin{itemize}
\item We propose a novel domain-transfer quality assessment method called DWIT-PCQA which can predict the perceptual quality of point clouds by leveraging the prior knowledge of images.

\item We decompose the conventional optimization objective of DA into two suboptimization problems: distortion-guided feature alignment and quality-aware feature disentanglement, which handles domain discrepancy in transfer-based quality assessment. 

% \item We propose a quality-aware feature disentanglement to preserve the quality-sensitive feature components in cross-domain feature alignment, addressing disruptions in quality mapping caused by distortion discrepancy.

\item  The proposed DWIT-PCQA shows reliable performance under training without point cloud annotations, which can promote the wider application of PCQA. %Additionally, further experiments demonstrate the scalability of DWIT-PCQA and the enhancement of the generalization ability for existing models. 

\end{itemize}

The remainder of this paper is organized as follows. The related work is discussed in Section \ref{sec:relatedwork}. Section \ref{sec:Reformulation} derives the functionality of each module by decomposing the conventional DA optimization objective. Section \ref{sec:network} presents the implementation of the proposed DWIT-PCQA, with its performance evaluation given in Section \ref{sec:experiments}. Finally, the conclusion is drawn in Section \ref{sec:conclusion}.

\section{Related Works}
\label{sec:relatedwork}

\subsection{Point Cloud Quality Assessment}
%\Note{SC: a little bit long.}
For PCQA, the FR metric is first studied due to its urgent requirement for point cloud compression (PCC). Moving Picture Experts Group (MPEG) has applied point-to-point (p2point)~\cite{Mekuria2016Evaluation}, point-to-plane (p2plane) ~\cite{tian2017geometric} and PSNRyuv~\cite{torlig2018novel} in PCC standardization. Further subjective experiment \cite{alexiou2017subjective,alexiou2017performance} demonstrate the unstable performance of these point-wise metrics when facing multiple types of distortion although with low computation complexity. Therefore, other PCQA metrics are proposed that take HVS-based features (e.g., structural features, curvature statistics and color lightness) into consideration and realize more robust performance \cite{meynet2020pcqm,yang2020inferring,Zhang2021MSGRAPHSIM,yang2021MPED,javaheri2021PTD,alexiou2018pointangular,javaheri2020haus,viol2020acolor,Alexiou2020pointSSIM, javaheri2021JPC}.

For NR-PCQA, \cite{Tao2021PMBVQA} employed multi-scale feature fusion to predict the quality of point clouds. \cite{Chetouani2021nrpcqa} split point clouds into local patches and used low-level patch-wise features to train the network. \cite{Liu2021PQANet} proposed to use multiview projection and distortion classification information to enhance quality prediction. \cite{Chai2024PlainPCQA} calculated the importance weights of projection patches based on information validity and constructed the quality description that connects 2D and 3D feature representations. \cite{liu2020LSPCQA} proposed to leverage pseudo MOSs and sparse convolution network to learn quality representation. \cite{Fan2022Video} and \cite{Zhang2022Video} introduced video quality assessment methods into PCQA by integrating point cloud projections into videos. \cite{shan2022GPANet} predicted the quality of point clouds using anti-perturbation features extracted from a graph neural network. \cite{Wu2024Hypergraph} introduced hypergraph learning into PCQA, leveraging the interactive information among vertices for quality representation. \cite{Zhang2022MMPCQA} leverage both point cloud projection and raw 3D data to extract integrated quality features. 

All the above methods require large-scale training data in the same scene as the test samples. This requirement is difficult to meet in many cases, such as for unknown testing scenes. In this work, we propose to use DA to connect the point cloud quality and the prior knowledge in other media formats (e.g., images). Considering the common judge of perceptual quality for different forms of media (e.g., image, video, and point cloud) applies to HVS, the prior knowledge in IQA can reveal the characteristics of HVS and contribute to the quality assessment of point clouds.

\subsection{Domain Adaptation}

DA aims to transfer knowledge from a labeled source domain to a related target domain, minimizing the need for costly labeled data in the target domain. DA typically aligns feature distributions between domains by minimizing discrepancies such as maximum mean discrepancy (MMD)~\cite{Tzeng2014MMD}, correlation alignment~\cite{Sun2016Corre} and Kullback Leibler divergence~\cite{Zhuang2015kl}. In recent years, adversarial DA, which reduces domain discrepancy through an adversarial objective concerning a domain discriminator ~\cite{Tzeng_2017_CVPR}, has become increasingly popular and is widely applied in visual tasks such as image classification and recognition  \cite{pmlr-v37-ganin15,NIPS2016_45fbc6d3}.

Currently, there are a few works that introduce DA into IQA. \cite{ChenDA} used MMD as a loss function to reduce domain discrepancy for screen content image quality prediction by treating natural images as the source domain. \cite{Chen2021DA} performed progressive alignment based on the divided confident and non-confident subdomains of target domains. \cite{Lu2022StyleAM} proposed the StyleAM to introduce style feature space into feature alignment for NR-IQA. \cite{Li2024DA} filtered source domain images based on the similarity between the source domain and the target domain for alignment but introduced the extra conversion error.

The image-to-point cloud transfer is a more challenging task. Some researches reveal the potential perceptual connection between images and point clouds, e.g., \cite{daICCVws} which used the bird's view of a point cloud image to enhance the original LiDAR point cloud data via DA, and \cite{PCRE,PCREG,Li2023Reconstruction} which studied 3D reconstruction and recognition based on single or multiview images. For image-to-point cloud transfer quality prediction, the distortion discrepancy makes the performance of the only existing attempt IT-PCQA \cite{Yang2022ITPCQA} which conducts the direct cross-domain feature alignment by DA unsatisfactory. Therefore, to handle the distortion discrepancies in image-to-point cloud transfer, we introduce the distortion-guided alignment in the DA optimization objective and meanwhile reduce the conversion error between distortion and perceptual quality using contrastive learning.

\section{Problem Formulation}
\label{sec:Reformulation}

In unsupervised domain adaption, the source and target domain are given as $\Pset\{(\x_i^s,\y_i^s)\}_{i=1}^{n_s}$ and $\Qset\{\x_j^t\}_{j=1}^{n_t}$, where $\x_i^s$ and $\x_j^t$ represent individual samples in the source and target domain, i.e., images and point clouds in this work, respectively.  $\y_i^s$ denotes the label of $\x_i^s$. We assume that the source and target domains are sampled from the distributions $\Dset_S(\x, \y)$ and $\Dset_T(\x, \y)$. 

For quality assessment, the ultimate goal of applying DA is to learn a \textbf{feature extractor} $G(\cdot)$ and a \textbf{regression module} $R(\cdot)$ to minimize the expected target risk, i.e.
\begin{align}
\label{eq0}
\varepsilon_{(\x_i^t,\y_i^t)\sim\Tset}[\Lset_{sim}\{R(G(\x_i^t)),\y_i^t)\}],
\end{align}
where $\Lset_{sim}\{\cdot\}$ is an index which can evaluate the discrepancy between the subjective and objective scores, such as Pearson linear correlation coefficient (PLCC), Spearman rank-order correlation coefficient (SROCC), Kendall rank-order correlation coefficient (KROCC) and root mean squared error (RMSE). 

Given $\z=G(\x)$, to achieve the goal of the domain-transfer quality assessment, confusing the feature distribution of the images $\Dset_S(\z)$ and the point clouds $\Dset_T(\z)$ is required, which can be achieved by aligning their  likelihood functions, i.e.
\begin{align}
\label{eq1}
\min||\mathcal{D}_S(\z \mid \y) - \mathcal{D}_T(\z \mid \y)||.
\end{align}

However, due to different low-level feature distributions directly caused by different distortion distributions, directly aligning these two likelihood functions is difficult. We resort to decomposing the optimization objectives to solve the issues.

%  (for example, under the scale of BT.500, the manifestation with the same MOS of 2.5 caused by JPEG distortion in images and V-PCC distortion in point clouds is significantly different)

\subsection{Optimization Objective Decomposition}

The preparatory experiment mentioned in Section \ref{sec:introduction} demonstrates the role of the distortion discrepancy in domain-transfer quality assessment. Specifically, the discrepancy between $\mathcal{D}_S(\z \mid \y)$ and $\mathcal{D}_T(\z \mid \y)$ can be explained by different low-level feature distributions related to different distortion manifestations. 
% which is exacerbated by different prior distortion distributions derived from mismatched distortion types between images and point clouds. 

% denoted as 
% \begin{align}
% \mathcal{D}_S(\z ) \neq \mathcal{D}_T(\z ),
% \end{align}

% quality scores are highly correlated with distortion types, e.g., $\y=f_c(\y_d)$ given the distortion types $\y_d$ and content-related mapping function $f_c$ \cite{liu2021csvt}, and assuming that if $G$ can generate a domain-invariant representation $\z=G(\x)$ to reduce the effect of content shift by extracting deep semantic features,

% 这里的偏差可以用对比学习表示

% ,
% \begin{align}
% \mathcal{D}_S(\y_{d}^{s}) \neq \mathcal{D}_T(\y_{d}^{t}).
% \end{align}

To reduce the transfer difficulty and emphasize common distortion patterns during cross-domain alignment, the prior distortion distribution is introduced into the DA optimization objective, i.e., Eq. \eqref{eq1}. As a result, the alignment target of Eq. \eqref{eq1} is decomposed into
\begin{align}\label{eq6.1}
\min||\mathcal{D}_S(\z \mid \y_d)-\mathcal{D}_T(\z \mid \y_d)||, 
\end{align}
and meanwhile regarding the biases in conversion from Eq. \eqref{eq1} to Eq. \eqref{eq6.1},
\begin{align}\label{eq5}
\min||\mathcal{D}_S(\z \mid \y)-\mathcal{D}_S(\z \mid \y_d)||. 
\end{align}

% According to \cite{Ben2010upperbound,Zhao2019upperbound,Tachet2020GLS}, 

\subsection{Sub Objective 1: Distortion-guided Feature Alignment}

Eq. \eqref{eq6.1} requires to align the feature distributions based on the distortion distributions between the source domain and the target domain. Since the unknown distortion distribution in the target domain is needed in Eq. \eqref{eq6.1}, a new \textbf{classification module} $H(\cdot)$ is first introduced to recognize the distortion types $\y_d=H(\z)$. According to Theorem 4.1 in \cite{Zhao2019upperbound}, under different prior distributions, the cross-domain alignment requires not only learning domain invariant features and achieving minimum source domain error, but also optimizing the error bias between the source and target domain, denoted as
\begin{align}\label{eq6.2}
||{\varepsilon _S} - {\varepsilon _T}|| = ||{\mathcal{D}_S}({\widehat{\y_d}} \ne {\y_d}) - {\mathcal{D}_T}({\widehat{\y_d}}\ne {\y_d})||,
\end{align}
which is related to the distance between the two distributions and the distance of the label mapping functions between the two domains \cite{Zhao2019upperbound}, and where $\widehat{\y_d}$ signifies the predicted distortion type. Specifically, the upper bound of Eq. \eqref{eq6.2} can be calculated according to the error decomposition theorem \cite{Ben2010upperbound,Tachet2020GLS,Zhao2019upperbound} as follows
\begin{align}\label{eq7}
\begin{split}
  &||{\mathcal{D}_S}({\widehat{\y_d}} \ne {\y_d}) - {\mathcal{D}_T}({\widehat{\y_d}}\ne {\y_d})|| \leqslant \\ &\| {\mathcal{D}_S({{\y_d}}) - \mathcal{D}_T({{\y_d}})} \|\mathop {\max }\limits_{{{\y}_d} } {{\Dset}_S}(\widehat {{{\y}_d}} \ne {{\y}_d}\mid {{\y}_d}) \\
   &+ 2(|{{\y}_d}| - 1)\mathop {\max }\limits_{{\widehat{\y_d}} \ne y_d } ||{{\Dset}_S}( {\widehat {{{\y}_d}} \ne \y_d \mid {{\y}_d}} ) - {{\Dset}_T}( {\widehat {{{\y}_d}} \ne \y_d \mid {{\y}_d}} )|| ,
   \end{split}
\end{align}
where $|{{\y}_d}|$ represents the number of distortion types. $\| {\mathcal{D}_S({{\y_d}}) - \mathcal{D}_T({{\y_d}})} \|$ is a constant associated with the dataset distribution. $\mathop {\max }\limits_{{{\y}_d} } {{\Dset}_S}(\widehat {{{\y}_d}} \ne {{\y}_d}\mid {{\y}_d})$ signifies the maximum probability that the predicted distortion type does not match the true one, which can be optimized by the optimal distortion classification in the source domain. And $\mathop {\max }\limits_{{\widehat{\y_d}} \ne \y_d } ||{{\Dset}_S}( {\widehat {{{\y}_d}} \ne \y_d \mid {{\y}_d}} ) - {{\Dset}_T}( {\widehat {{{\y}_d}} \ne \y_d \mid {{\y}_d}} )|| $ measures the distance of the conditional distortion distribution $\widehat{\y_d}|\y_d$ between the source domain and the target domain, which leads to the only unknown distortion distribution in the target domain $\Dset_T(\y_d)$ that needs further processing.

To deal with the unknown distortion distribution in the target domain $\Dset_T(\y_d)$, a distribution weight is defined as
\begin{align}
\label{eq8}
\mathbf{w}_y=\frac{\mathcal{D}_T(\y_d)}{\mathcal{D}_S(\y_d)}.
\end{align}

Then according to the clustering structure theorem \cite{Tachet2020GLS}, 
$\mathop {\max }\limits_{{\widehat{\y_d}} \ne \y_d } ||{{\Dset}_S}( {\widehat {{{\y}_d}} \ne \y_d \mid {{\y}_d}} ) - {{\Dset}_T}( {\widehat {{{\y}_d}} \ne \y_d \mid {{\y}_d}} )|| $ reaches optimality as long as $\mathbf{w}_y$ can achieve a partition of $\z = { \cup _{{\y_d}}}{\z_{{\y_d}}}$ such that $\forall {\y_d}$, ${\Dset_S}(\z \in \z_{{\y_d}}|{\y_d}) = {\Dset_T}(\z \in \z_{{\y_d}}|{\y_d}) = 1$, which is consistent with the optimization objective of the classifier $H(\cdot)$. 

When Eq. \eqref{eq6.1} reaches optimum by ensuring the optimal $H(\cdot)$, we obtain
\begin{align}\label{eq8.1}
\mathcal{D}_T(\z)=\sum_{y_d \in \mathcal{Y}_d} \mathbf{w}_y \cdot \mathcal{D}_S(\z, \y_d)= \mathcal{D}_S^w(\z),
\end{align}
which demonstrates that to achieve Eq. \eqref{eq6.1}, aligning the \textbf{re-weighted feature distribution} of the source domain $\Dset_S^w(\z)$ based on the distortion distribution and the feature distribution of the target domain $\Dset_T(\z)$ is requred.

\subsection{Sub Objective 2: Quality-sensitive Maintenance}

Eq. \eqref{eq5} aligns the feature likelihood function with respect to quality $\mathcal{D}_S(\z \mid \y)$ and distortions $\mathcal{D}_S(\z \mid \y_d)$, which physically requires \textbf{consistent feature distributions} for quality prediction and distortion recognition. In other words, Eq. \eqref{eq5} requires to make the features $\z_i$ both quality-aware and distortion-aware, denoted as
\begin{align}
\min \sum\limits_i {||{z_i} - {z_i^q}|| + ||{z_i} - {z_i^d}||},
\end{align}
where $z_i^q$ indicates the quality-aware representation, and $z_i^d$ indicates the distortion-aware representation. 

To achieve this, we resort to constractive learning, denoted as 
\begin{align}
\min \sum\limits_{i,j \in {P_{q,d}}} {||{z_i} - {z_j}||}  + \max \sum\limits_{i,j \in {N_{q,d}}} {||{z_i} - {z_j}||},
\end{align}
where the positive samples $P_{q,d}$ satisfy both the quality-aware representation and distortion-aware representation and vice versa for negative samples $N_{q,d}$. Specifically, it aims to distill the quality-sensitive components related to common distortions among extracted features. The original features $\z=G(\x)$ are considered to contain the quality-aware and quality-unaware components. Further, the quality-aware components can be divided into distortion-related and distortion-unrelated components. If the distortion-related quality-aware feature components are extracted, the extracted feature distribution can be well applied to both quality prediction and distortion recognition although consuming a certain degree of fitting ability.

% \subsection{Derived Modules}

% Based on the above derivation, several modules can be identified. The proposed method includes the following modules:

%As a result, with \textbf{the harmonized likelihood functions of features with respect to distortions and perceptual quality} (referring to Eq. \eqref{eq5}) and the guaranteed optimality of classifier $H(\cdot)$ (referring to Eq. \eqref{eq7}), instead of solely aligning the feature distribution of the source domain $\Dset_S(\z)$ and the target domain $\Dset_T(\z)$, 

% The alignment operation can be achieved by the adversarial DA \cite{Ganin2016DANN,Long2017JAN,Long2018CDAN,Yang2022ITPCQA} which confuses a proposed \textbf{conditional domain discriminator} $D(\cdot)$ to not distinguish target domain samples from source domains with a distribution-weighted cross-entropy (DWCE) loss denoted as
% \begin{align}
% \label{eq10}
% \Lset_{DWCE} = -\mathbb{E}_{\x_i^s\sim\Dset_S}\w_{y_i}log[D(\z_i^s)]- \mathbb{E}_{\x_j^t\sim\Dset_T}log[1-D(\z_j^t)].
% \end{align}

\section{Methodology}\label{sec:network}

% As we have discussed above, the bridge between the point cloud media and the image media should consider their different prior distortion distribution, i.e. the mismatched $\mathcal{D}_S(\y_d^{s}) $ and $ \mathcal{D}_T(\y_d^{t})$, deriving from different distortion types in the two media. To address it, we decompose the optimization objective of the conventional DA (referring to Eq. \eqref{eq1}) into two suboptimization functions with distortion as a transition (referring to Eq. \eqref{eq6.1} and Eq. \eqref{eq5}), and implement it by a proposed distortion-guided DA framework to achieve the media-transfer NR quality assessment. 

The proposed DWIT-PCQA is illustrated in Figure \ref{fig:scheme}. Specifically, it has the point cloud preprocessing module, the feature generative network $G$, the quality-aware feature disentanglement module, the conditional-discriminative network $D$, the distortion classification network $H$, and the quality regression network $R$.

\begin{figure*}[pt]
\setlength{\abovecaptionskip}{0.cm}
\setlength{\belowcaptionskip}{-0.cm}
\centering
\includegraphics[width=1\linewidth]{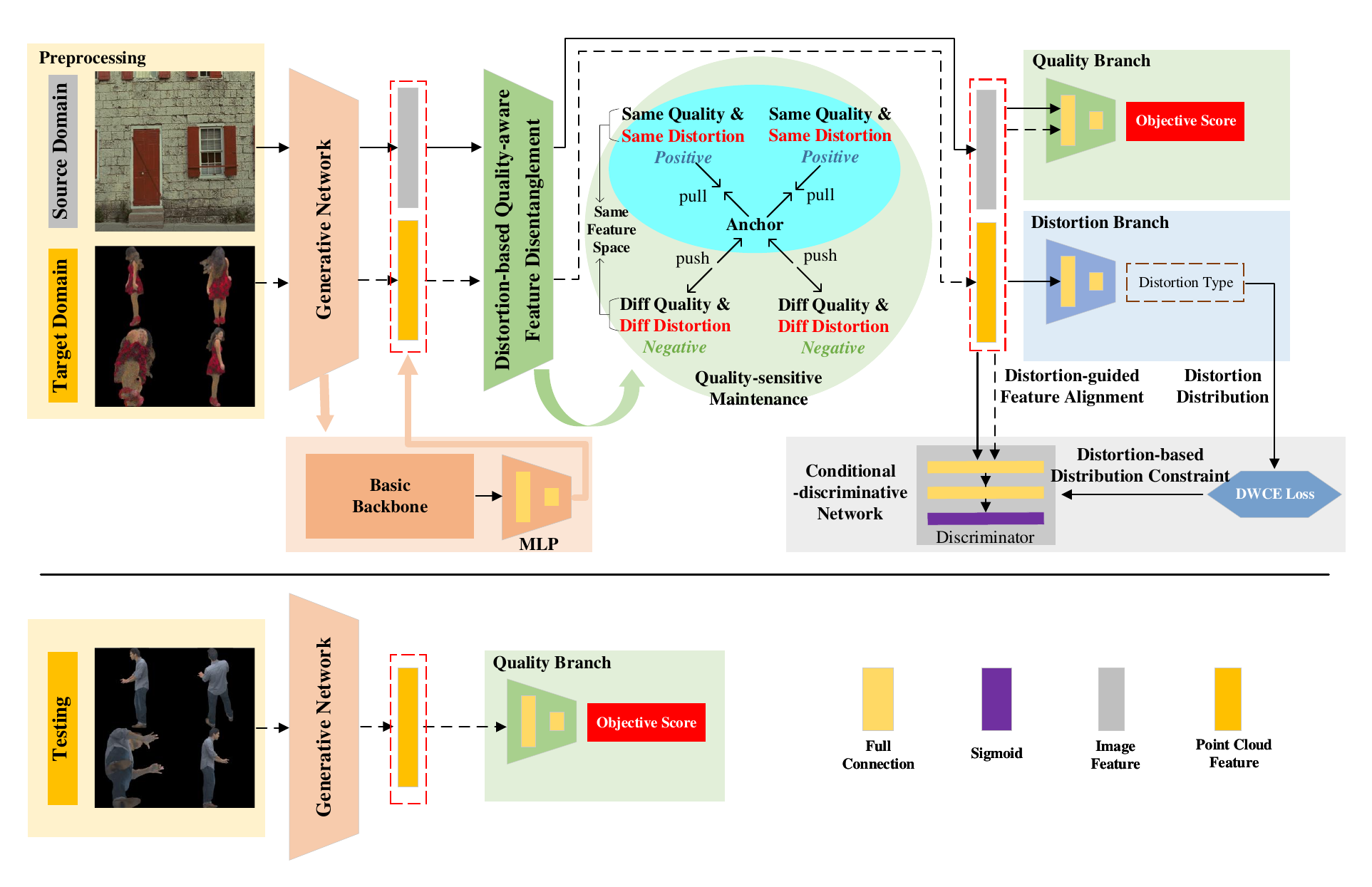}
\caption{An illustration of our proposed DWIT-PCQA. It consists of the following substeps: (a) the preprocessing module casts the point clouds and images into the same data format; (b) the generative network extracts representative features from both source and target domains; (c) the conditional-discriminative network aligns the feature likelihood functions with respect to distortions $\Dset(\z|\y_d)$ between the source domain and the target domain; (d) the quality-aware feature disentanglement module is applied to harmonize the conflicts of feature likelihood functions with respect to distortions $\Dset(\z|\y_d)$ and perceptual quality $\Dset(\z|\y)$; (e) the distortion classification network estimates the necessary distortion distribution for cross-domain feature alignment under biased distortions; (f) the quality regression network achieves quality prediction. }
\label{fig:scheme}
\end{figure*}

The preprocessing module is used to transfer point clouds to multi-perspective images, casting point clouds into the same data format as images. The feature generative network $G$ is used to extract the representative features for both the source domain and the target domain. Then based on the analysis in Section \ref{sec:Reformulation}, to handle the distortion discrepancy, we propose a distortion-based conditional-discriminative network $D$ to promote minimizing the feature likelihood function discrepancies with respect to distortions $\Dset(\z|\y_d)$, which aligns the feature distributions $\mathcal{D}_S^w(\z)$ re-weighted based on distortion distributions. The distortion distributions in the target domain are estimated by an introduced distortion classification branch $H$. Besides, we propose a quality-aware feature disentanglement to mitigate the destruction of the feature likelihood function with respect to quality $\Dset(\z|\y)$ during aligning the feature likelihood function with respect to distortions $\Dset(\z|\y_d)$ and to maintain the quality-sensitive representation during alignment with biased distortions. It is accomplished by contrastive learning with designed positive and negative pairs. The quality regression network $R$ predicts the final objective scores based on the aligned features.

% Besides, for the discriminator $D$, we realize that only minimizing the domain discrepancy is not enough for cross-media quality assessment. Therefore, we propose a conditional-discriminative network that employs the distribution weights to handle the unbalanced distortion distributions to regress an accurate objective score.

 \subsection{Feature Extraction For Two Domains}
 \label{sec:featureextraction} 
 %\Note{SC: the logic flow in this subsection is very confusing. involve confusing subjective experiments.}

To share a common feature encoder, the images and point clouds need to be cast into the same data format. Therefore, similar to \cite{Yang2022ITPCQA}, we first project the 3D point cloud onto the perpendicular planes of a cube. Then, the side projections are spliced together to form a multi-perspective image as the network input, as shown in Figure \ref{fig:scheme}. Images and spliced perspective views are resized into $224\times224$.

The feature generative network contains a basic backbone and a multilayer perceptron (MLP).  A weight-shared ResNet-50 backbone is adopted to extract the output features from the last layer for both the source domain and the target domain. Then the basic backbone is followed by the MLP which contains 2 fully-connected (FC) layers with channels $[2048,1024]$ and $[1024,256]$, to adjust the dimension of output features from 2048 to 256. This feature generative network is indicated as $G(\cdot)$.

\subsection{Distortion-guided Feature Alignment}
\label{sec:conditional}
% \Note{SC: what is the overall sketch of this design?}

The feature alignment can be conducted by fully confusing a discriminator so that it cannot distinguish whether the generated features are from the source domain, i.e., images, or the target domain, i.e., point clouds, leading to a similar cross-domain feature distribution. However, the bias in the prior distortion distribution suppresses the confusion effect for the discriminator. As analyzed in Section \ref{sec:Reformulation}, additional feature distribution re-weighting based on the prior distortion distribution is required between different domains.

% Considering the relation between distortion and perceptual quality, forcing alignment may destroy the likelihood functions of features with respect to perceptual quality, suppressing the quality-sensitive representation.

\subsubsection{Distortion Weighted Cross Entropy}

To minimize domain discrepancy while considering distortion distributions, we design a new discriminator $D$, i.e., the conditional-discriminative network, where a distortion-based weight is incorporated to help align the unbalanced feature distribution with biased prior distortion distributions. The weighting factor introduced is to assign feature distributions with different weights based on the marginal distortion distributions $\mathcal{D}_S(\y_d)$ and $\mathcal{D}_T(\y_d)$, forming a distribution-weighted cross-entropy (DWCE) loss. 

% While reducing the domain discrepancy, we attempt to boost aligning for unbalanced feature distributions with different prior distortion distributions.

Specifically, the output of the feature generative network, i.e., $\z=G(\x)$, is fed into a discriminator $D$. Considering that the unbalanced feature distribution should lead to the unbalanced feature alignment, we reward/punish the features of overlapping/non-overlapping distortion to tackle the biased marginal distortion distribution, which is formulated as
\begin{align}\label{eq:CCEL}
\begin{split}
\Lset_{DWCE} = -\mathbb{E}_{\x_i^s\sim\Dset_S}\w_{y_i}log[D(\z_i^s)]- \mathbb{E}_{\x_j^t\sim\Dset_T}log[1-D(\z_j^t)],
\end{split}
\end{align}
where $\w_{y}=\frac{\mathcal{D}_T(\y_d)}{\mathcal{D}_S(\y_d)}$ is the distribution weight which signifies the cross-correlation of the marginal distortion distribution between the source domain $\mathcal{D}_S(\y_d)$ and the target domain $\mathcal{D}_T(\y_d)$. $\w_{y}$ can be estimated based on the predicted results $\hat{\y_d}$ in the target domain according to \textbf{Theorem 1}.

\begin{theorem}
\label{theorem1}
If $\mathcal{D}_S(\z \mid \y_d)=\mathcal{D}_T(\z \mid \y_d)$ and $\z \rightarrow \y_d$ follows the same mapping $H$ on $S$ and $T$, $\mathcal{D}_T(\hat{\y_d})=\mathcal{D}_S(\hat{\y_d}, \y_d) \frac{\mathcal{D}_T(\y_d)}{\mathcal{D}_S(\y_d)}$.
\end{theorem}
\begin{proof}
\begin{align}
\begin{split}
\mathcal{D}_T(\hat{\y_d} \mid \y_d)&= \mathcal{D}_T(\hat{\y_d} \mid \z, \y_d) \mathcal{D}_T(\z \mid \y_d)\\
& = \mathcal{D}_T(\hat{\y_d} \mid \z, \y_d) \mathcal{D}_S(\z \mid \y_d) \\
& = \mathcal{D}_S^H(\hat{\y_d} \mid \z) \mathcal{D}_S(\z \mid \y_d)\\
& = \mathcal{D}_S(\hat{\y_d} \mid \z, \y_d) \mathcal{D}_S(\z \mid \y_d)=\mathcal{D}_S(\hat{\y_d} \mid \y_d)
\end{split}
\end{align}
\begin{align}
\begin{split}
\mathcal{D}_T(\hat{\y_d}) & = \mathcal{D}_T(\hat{\y_d} \mid \y_d) \mathcal{D}_T(\y_d) \\
& = \mathcal{D}_S(\hat{\y_d} \mid \y_d) \mathcal{D}_T(\y_d)= \mathcal{D}_S(\hat{\y_d}, \y_d) \frac{\mathcal{D}_T(\y_d)}{\mathcal{D}_S(\y_d)}
\end{split}
\end{align}
\end{proof}

\textbf{Theorem}~\ref{theorem1} explains that $\w_y$ can be calculated based on the predicted distortion results $\widehat{\y_d}$ in the target domain under the optimal classifier $H(\cdot)$. Defining $\mathbf{C}$ as the discrete binary distribution of $H(\cdot)$, 
\begin{align}
\label{eq11}
\mathbf{C}=\mathcal{D}_S(\widehat{\y_d}, \y_d),
\end{align}
and $\boldsymbol{\mu}$ as the distribution of predictions in the target domain,
\begin{align}
\label{eq12}
\boldsymbol{\mu}=\mathcal{D}_T(\widehat{\y_d}),
\end{align}
$\w$ is calculated by $\mathbf{w}_y=\mathbf{C}^{-1} \boldsymbol{\mu}$, which can be approximately solved by the quadratic program \cite{Lipton2018Labelshift}, 
\begin{align}\label{eqQP}
\begin{array}{ll}
\underset{\mathbf{w}}{\operatorname{minimize}} & \frac{1}{2}\|\hat{\boldsymbol{\mu}}-\hat{\mathbf{C}} \mathbf{w}_y\|_2^2 \\
\text {subject to} & \mathbf{w}_y \geq 0, \mathbf{w}_y^T \mathcal{D}_S(\y_d)=1 .
\end{array}
\end{align}
where $\hat{\mathbf{C}}$ and $\hat{\boldsymbol{\mu}}$ signify the estimation of ${\mathbf{C}}$ and ${\boldsymbol{\mu}}$ with finite samples. \textbf{Algorithm \ref{alg:IWcompute}} details the calculation of $\w_{y}$.

\renewcommand{\algorithmicensure}{\textbf{Output:}}
\begin{algorithm}
\caption{Distortion-based Distribution Weight Computation}
\label{alg:IWcompute}
\begin{algorithmic}
\Require distortion labels in the source domain $\y_{d}^{s}$, predicted distortion results in the target domain $\widehat{\y_{d}^{t}}$
\Ensure distortion-based distribution weight $\w_{y}$
\State $\w_{y} \leftarrow 1$
\For {$epoch = 1$ to EpochNum} 
\State Compute the distortion distribution in the source domain ${\y_{d}^{s}}^{'}=normalize(sum(onehot(\y_{d}^{s})))$
\State Compute the distortion distribution in the target domain $\widehat{{\y_{d}^{t}}^{'}}=normalize(sum(onehot(\widehat{\y_{d}^{t}})))$
\State Obtain the distortion classification results in the target domain $\hat{\boldsymbol{\mu}}=\widehat{{\y_{d}^{t}}^{'}}$
\State Compute the discrete binary distribution of the distortion classification results in the source domain and the target domain $\hat{\mathbf{C}}=\widehat{{\y_{d}^{t}}^{'}}^T{\y_{d}^{s}}^{'}$
\State Update $\w_{y}$ by quadratic program $\w_{y}=\hat{\mathbf{C}}^{-1} \hat{\boldsymbol{\mu}} \leftarrow QP(\hat{\mathbf{C}},\hat{\boldsymbol{\mu}})$
\EndFor
\end{algorithmic}
\end{algorithm}

\subsubsection{Physical Meaning}

Physically, samples with close distortions in both the source and target domains, such as Gaussian noise and color noise, are emphasized to promote learning of cross-domain commonalities (e.g., the higher weights are given for distortion types with higher proportions in the target domain). Samples with unique image distortions in the source domain, such as JPEG2000 lossy compression, are punished to reduce the impact of divergences between different domains on training (e.g., the weights approaching 0 are given for distortions not present in the target domain). The samples with unique point cloud distortions in the target domain, such as V-PCC lossy compression, are still involved in feature alignment, since the target domain is a black box and only the distortions in the source domain can be known. These point cloud distortions are actually tried to be expressed as the combination of source domain distortions through $H$, ensuring effective utilization of information in the target domain. 

% For distortions both in the source domain and the target domain, the samples in the source domain will be imposed a relatively larger weight if the proportion of samples in the target domain with the same distortion is larger than that in the source domain. Correspondingly, the samples in the source domain will be imposed a relatively smaller weight if the proportion of samples in the target domain with the same distortion is smaller than that in the source domain. And the samples in the source domain will not be involved in training if their distortions do not exist in the target domain.

% \State Update the distortion-based distribution weight $\w_{y}\leftarrow\lambda QP(\hat{\mathbf{C}},\hat{\boldsymbol{\mu}})+(1-\lambda) \w_{y_s}$

\begin{figure*}[t]
\setlength{\abovecaptionskip}{0.cm}
\setlength{\belowcaptionskip}{-0.cm}
	\centering
		\includegraphics[width=0.8\linewidth]{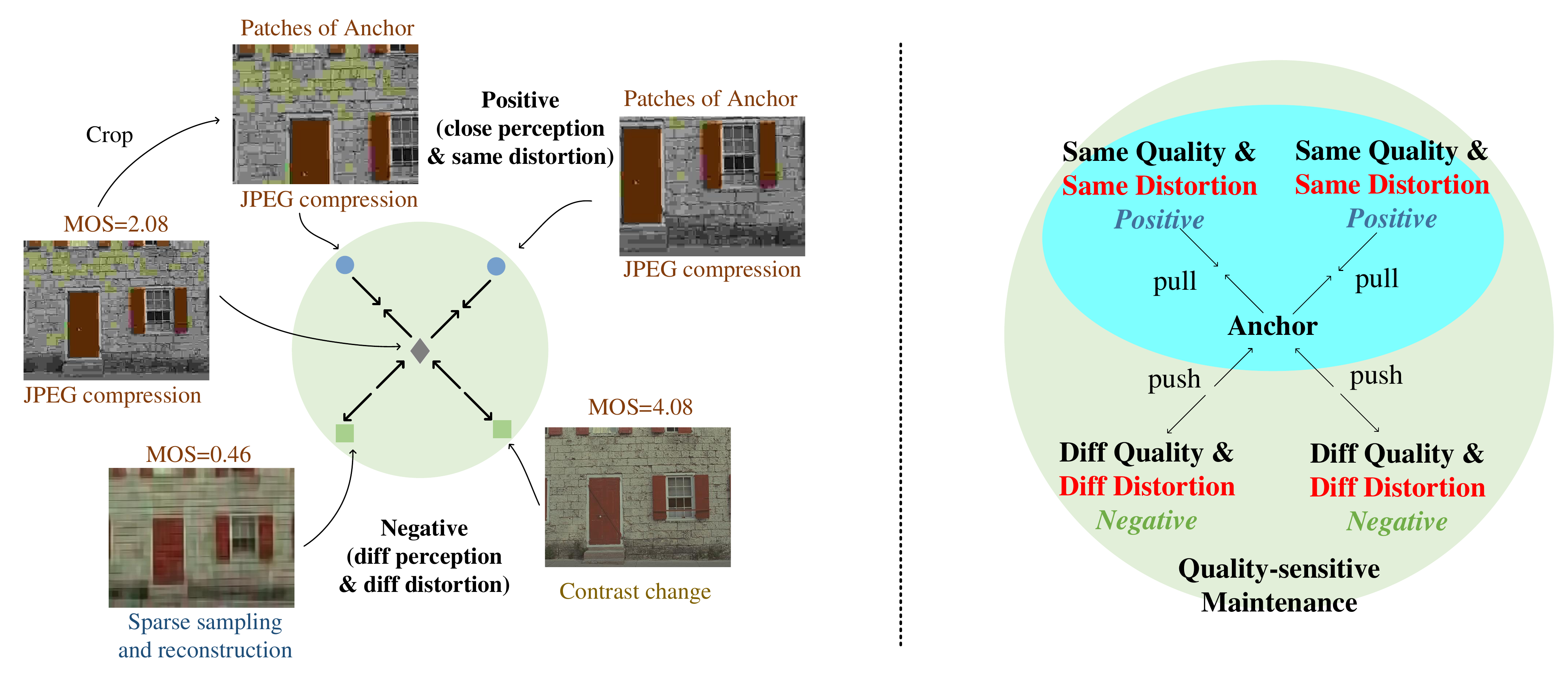}
	\caption{Positive and negative sample configuration for the quality-aware feature disentanglement. Positive and negative samples are designed to simultaneously achieve two distributions related to quality and distortion in the same feature space. The patches of the same image are set to be positive samples of anchor images (with close perception), and the images with the same content but different distortion types are set to be negative samples (with different perception). In this way, the extracted features are promoted to be both quality-ware and distortion-aware.}
	\label{fig:posandneg}
\end{figure*}

\subsection{Quality-sensitive Component Maintenance}
\label{sec:disentanglement}

To solve the offset of the prior distortion distribution of two media, distortion-guided alignment is conducted as section \ref{sec:conditional} proposed. However, considering that the perceptual quality is not equivalent to the distortion, forcing alignment of different feature distributions with respect to distortions $\Dset(\z|\y_d)$ between two media may destroy the feature likelihood functions with respect to perceptual quality $\Dset(\z|\y)$. As a result, the learning of cross-domain commonality suppresses the quality-sensitive representation, i.e., the required feature distributions for distortion recognition and quality assessment are not completely consistent. Therefore, to achieve Eq. \eqref{eq5}, we propose a quality-aware feature disentanglement to harmonize the conflict and enhance the quality-sensitive representation. 

% However, the requirements of features for distortion recognition are usually low-level and coarse-grained, whereas those for quality prediction are usually high-level and fine-grained. It is contradictory to expect to mix these features in learning simultaneously: the forced alignment operation for two discrepant domains may destroy the quality-sensitive information in the features.

% Considering that the perceptual quality is highly related to the distortions, the alignment destroys the likelihood distribution of features with respect to perceptual quality, resulting in difficulties in cross-domain quality-sensitive fitting.

The extracted features are considered to contain quality-aware and quality-unaware components. The quality-aware components can be divided into distortion-related and distortion-unrelated components. The proposed feature disentanglement module aims to extract distortion-related quality-aware features. In this way, the extracted feature distribution can be well applied to both the quality branch and distortion branch, although consuming a certain degree of fitting ability. To achieve this, we resort to contrastive learning to restrict the feature distributions related to both distortion and quality in the same feature space, as shown in Figure \ref{fig:posandneg}.

\subsubsection{Positive and Negative Sample Setting}

Suppose that the source domain has $N$ original images, represented by $\x^{s(1)},\x^{s(2)}, \cdots, \x^{s(N)}$. Each original image $\x^{s(n)}$ is degraded by $K$ distortion types, resulting in $\x^{s(n)}_{d_1},\x^{s(n)}_{d_2}, \cdots, \x^{s(n)}_{d_K}$. Then we extract $M$ patches from $\x^{s(n)}_{d_k}$ but with different locations, denoted $\x^{s(n)}_{d_k,1},\x^{s(n)}_{d_k,2}, \cdots, \x^{s(n)}_{d_k,M}$. The patches $\x^{s(n)}_{d_k,1},\x^{s(n)}_{d_k,2}, \cdots, \x^{s(n)}_{d_k,M}$ are considered as positive samples of the source image $\x^{s(n)}_{d_k}$ due to their almost identical perceived quality derived from consistent content and distortion type. And the whole images with different distortion types, denoted $\x^{s(n)}_{d_j}$, and meanwhile with different perceptual quality are considered as the negative samples of the anchor image $\x^{s(n)}_{d_k}$. Note that the same content and different distortion often result in different perception, so the symbol for negative samples is directly labeled as $\x^{s(n)}_{d_j}$ without any special markings.

%due to consistent content but different distortion types, resulting in different perceptual quality. 

The following matters are worth mentioning. First, patches $\x^{s(n)}_{d_k,m}$ and their whole images $\x^{s(n)}_{d_k}$ are fed together into the network to ensure that the network is able to obtain complete positive sample images for training. Second, to ensure that the perceptual quality among positive samples is as consistent as possible, the cropped patches are set to have a sufficient overlap area. 

%Third, the negative pairs with different distortion types but similar quality scores of a small probability are removed to ensure that significant quality differences exist between the anchors and the negative samples.

\subsubsection{Physical Meaning}

As a result, the proposed feature disentanglement simultaneously implements two constraints on the feature distribution, leading to simultaneously satisfying two distributions on the same feature space. First, samples with similar quality are brought closer, while samples with different quality are pulled farther away. Meanwhile, the samples with the same distortion are clustered and vice versa. In this way, the extracted features are promoted to be meanwhile quality-aware and distortion-aware.

% Third, although there may exist some noise among negative samples, it is negligible due to the very small proportions of the images with different distortion types but the same quality scores.

\subsubsection{Loss Function}

To promote quality-aware learning, we employ infoNCE loss \cite{Oord2018InfoNCE}. Let $G(\cdot)$ denote the feature generative network. Given the images $\x$ from the source domain, the extracted features can be represented by $\z=G(\x)$. After defining the similarity between two extracted features, i.e.
\begin{align}
h( {{{\mathbf{z}}_1},{{\mathbf{z}}_2}} ) = \exp ( {\frac{{{{\mathbf{z}}_1} \cdot {{\mathbf{z}}_2}}}{{{{\| {{{\mathbf{z}}_1}} \|}_2} \cdot {{\| {{{\mathbf{z}}_2}} \|}_2}}} \cdot \frac{1}{\tau }} ),
\end{align}
where $\tau$ is a hyper-parameter that controls the range of the results, we implement the InfoNCE loss to regularize the distribution of extracted features. The contrastive loss function for a positive pair with the anchor feature ${{\mathbf{z}}_{{d_k}}^{s(n)}}$ is formulated as:
\begin{equation}
\mathcal{L}_{dis}^{n,m} = - \log [ \frac{{h( \mathbf{z}_{d_k}^{s(n)},\mathbf{z}_{d_k,{m}}^{s(n)} )}}{{h( \mathbf{z}_{d_k}^{s(n)},\mathbf{z}_{d_k,{m}}^{s(n)} ) + \sum_{j \ne k} h ( \mathbf{z}_{d_k}^{s(n)},\mathbf{z}_{d_j}^{s(n)} )}} ],
\end{equation}
where $\mathbf{z}_{d_k,{m}}^{s(n)}$ signifies a cropped version of ${{\mathbf{z}}_{{d_k}}^{s(n)}}$, demonstrating different scales, and ${{\mathbf{z}}_{{d_j}}^{s(n)}}$ refers to a different distorted version of ${{\mathbf{z}}_{{d_k}}^{s(n)}}$ with the consistent content but different distortion type and different perception.

For a mini-batch, the loss function to regularize the feature distribution is as follows:
\begin{equation}
\mathcal{L}_{Fea} = \frac{1}{{n_sm_s}}\sum_{n = 1}^{n_s} \sum_{{m = 0}}^{m_s} {\mathcal{L}_{dis}^{n,m}},
\end{equation}
where $m_s$ represents the patch number, and $n_s$ represents the sample number in the source domain.

\subsection{Distortion Distribution Prediction}
\label{sec:classification}

% The distortion information is required to handle the prior distribution bias derived from the mismatched distortion types. According to the analysis in Section \ref{sec:problem_formulation}, with the unknown prior distribution of the target domain, the predicted distortion information of the target domain can be leveraged. Therefore,

Since the distortion distribution in the target domain is involved in the calculation of the similarity weight $\w_{y}$ in Eq. \eqref{eq8.1}, a distortion classification network $H$ is introduced to provide the necessary distortion distribution. $H$ contains twofold FC layers, i.e., channels $[2048,1024]$ and $[1024,n_d]$ where $n_d$ denotes the distortion types in the source domain, to predict the probability belonging to each distortion type. % The number of distortion types is consistent with the samples in the source domain.

\subsubsection{Training by Source Domain}

The classification network is trained by the distortion annotation in the source domain. Given a set of images in the source domain denoted $\{\x^s_i, \y^{s}_{d_i}\}^{n_s}_{i=1}$ where $\y^s_{d_i}$ signifies the distortion type, we can obtain a latent feature $\z$ for each image via the feature generative network, i.e., $\{\x^s_i|\z_i^s, \y^s_{d_i}\}^{n_s}_{i=1}$, $\z_i^s=G(\x_i^s)$. We train the classification network via a cross-entropy loss function:
\begin{align}
{\mathcal{L}_{Cls}} =  - \sum\limits_{i = 1}^{{n_s}} {{\y_{d_i}^s}\log (\widehat{\y_{d_i}^s})},
\end{align}
where $\widehat{\y_{d_i}^s}=H(\z_i^s)$ denotes the predicted distortion type of samples in the source domain.

\subsubsection{Estimating Distortion Distribution in Target Domain}

The estimated distortion distribution obtained through $H$ in the target domain can be used to calculate the required weight $\w_{y}=\frac{\mathcal{D}_T(\y_d)}{\mathcal{D}_S(\y_d)}$ in Eq. \eqref{eq8.1} based on \textbf{Theorem}~\ref{theorem1} following \textbf{ Algorithm \ref{alg:IWcompute}}.

% \subsubsection{Distortion Distribution Calculation and Estimation}

% The distortion distribution in the source domain is known and can be obtained by normalizing the one-hot vector of distortion classification, denoted as 
% \begin{align}
% \mathcal{D}_S(\y_d)=normalize(sum(onehot(\y_{d}^{s}))).
% \end{align}

% The distortion distribution in the target domain can be estimated through the output results of the classification network, denoted as
% \begin{align}
% \mathcal{D}_T(\y_d)=normalize(sum(onehot(\widehat{\y_{d}^{t}}))),
% \end{align}
% where $\widehat{\y_{d}^t}=H(\z^t)$ is the predicted results of the network for features $\z^t$ in the target domain.

\subsection{Quality Regression}\label{sec:r}
We use a quality regression network 
$R$ containing twofold FC layers, i.e., channels $[2048,1024]$ and $[1024,1]$, to regress an objective score of the aligned latent feature generated jointly by the feature generative network $G$, the quality-aware feature disentanglement module, and conditional-discriminative network $D$.

The regression network is also trained using quality annotations in the source domain. Given a set of features of distorted images from the source domain, i.e., $\{\x^s_i|\z_i^s, \y_i^s\}^{n_s}_{i=1}$, $\z_i^s=G(\x_i^s)$ where $\y^s_{i}$ signifies the MOS of the distorted samples, the regression network is trained via
\begin{equation}\label{eq:r}
\Lset_{Reg}=\frac{1}{n_s}\sum_{i}^{n_s}(\widehat{\y_i^s}-\y^s_i)^2,
\end{equation}
where $\widehat{\y_i^s} = R(\z_i^s)$ denotes the predicted quality score of samples in the source domain. 

During testing, the quality scores of the samples $\x^t_i$ in the target domain can be directly predicted by $\widehat{\y_i^t} = R(\z_i^t)$ as shown in the bottom row of Fig. \ref{fig:scheme}.

\subsection{Overall Loss}

The overall loss is:
\begin{align}
\mathcal{L} = \lambda_1\mathcal{L}_{Fea} + \lambda_2\mathcal{L}_{Cls} + \lambda_3\mathcal{L}_{DWCE} + \lambda_4\mathcal{L}_{Reg},
\end{align}
where the parameters $\lambda_1$, $\lambda_2$, $\lambda_3$ and $\lambda_4$ are the weighting factors.

\begin{figure*}[pt]
\centering
\includegraphics[width=1\linewidth]{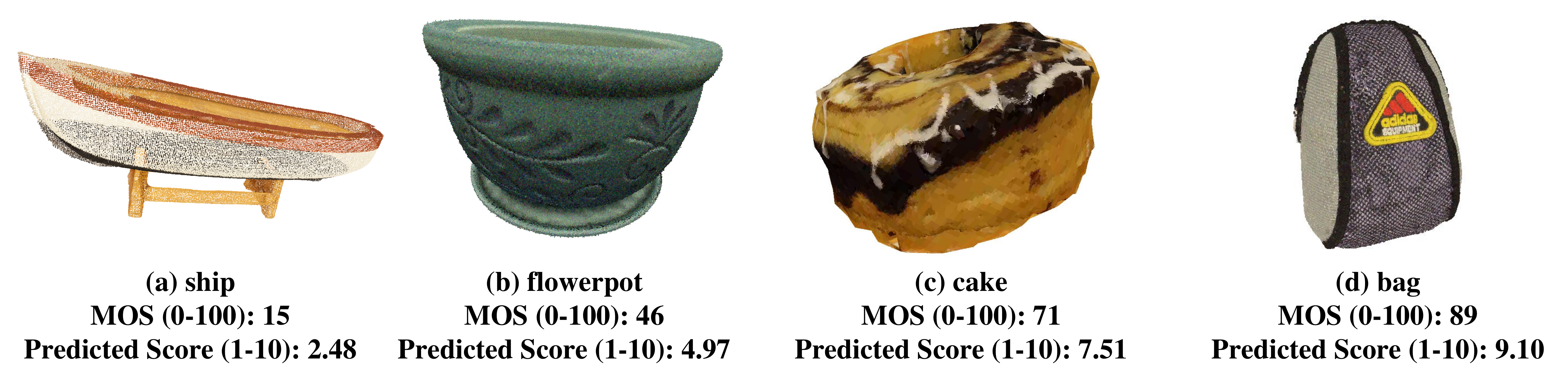}
\caption{Exemplary point clouds on the WPC dataset with the subjective MOS, the predicted quality score of the proposed DWIT-PCQA (KADID as the source domain and WPC as the target domain). The predicted quality score aligns with subjective perception even without using labels from the target domain for training.}
\label{fig:examples}
\end{figure*}

\section{Experiment}\label{sec:experiments}

\subsection{Implementation Details}

Our experiments are performed on NVIDIA 3090 GPUs. To test the proposed method, we use four independent databases, including two for images (TID2013 \cite{tid2013} and KADID-10k \cite{Lin2019KIAID10k}) and two for point clouds (SJTU-PCQA \cite{yang2020predicting} and WPC \cite{Su2019WPC}).

To ensure consistency between predicted scores and MOSs, a nonlinear Logistic-4 regression is applied to map the predicted scores to the same dynamic range, following the recommendations suggested by the Video Quality Experts Group (VQEG) \cite{Antkowiak2000vqeg,2006logistic5}. The employed evaluation metrics include PLCC and RMSE for prediction accuracy, SROCC and KROCC for prediction monotonicity. The higher PLCC, SROCC, or KROCC values indicate better performance, while the lower RMSE reflects better performance.

We implement our model using PyTorch. The ResNet-50 backbone is initialized with the pre-trained model on ImageNet \cite{Deng2009ImageNet}. The Adam optimizer is employed with a weight decay of 1e-4, and an initial learning rate of 5e-5. Natural and projected images are resized to $224\times224\times3$. The batch size for the source domain is set to 20, and the batch size for the target domain is set to 35. To implement feature disentanglement, both of the positive and negative sample numbers for an anchor are set to 2. The default training duration is set to 300 epochs. All weighting factors in the overall loss are set to 1.

For NR methods, following training-testing setups in~\cite{shan2022GPANet}, 5-fold cross-validation is adopted for the PCQA datasets (SJTU-PCQA and WPC) to reduce content bias. For SJTU-PCQA, the dataset is split into the training-testing sets with the ratio of 7:2 according to the reference point clouds for each fold, and the average performance in the testing sets is recorded. Similarly, for WPC, the training-testing ratio is set to 4:1. For FR metrics, their performance in the testing set for each fold is averaged and recorded. 

For image-to-point cloud transfer evaluation, the whole IQA datasets (e.g., TID2013 and KADID-10k) are used as the source domain for training. The PCQA datasets (e.g., SJTU-PCQA and WPC) follow the same training-testing partition used as the target domain during training and the testing set respectively, but the difference compared to general methods is that the quality labels of point clouds are not used in training and only the data itself is used for feature alignment.

\subsection{Image-to-point cloud Transfer Performance Comparison with Existing Methods}

Since images and point cloud projections share the same 2D format, IQA methods can also be used for image-to-point cloud transfer. In this section, the image-to-point cloud transfer performance of the proposed DWIT-PCQA is compared with the existing general IQA methods (e.g., DBCNN \cite{SCNN}, HyperIQA \cite{Su2020Hyper}), existing image-to-image transfer IQA methods (e.g., StyleAM \cite{Lu2022StyleAM} and Chen's method \cite{ChenDA}) and existing image-to-point cloud transfer methods IT-PCQA \cite{Yang2022ITPCQA}. The backbone of IT-PCQA is set to general SCNN \cite{SCNN}. Their backbones are initialized with the pre-trained model on ImageNet. The IQA datasets (e.g., TID2013 and KADID-10k) are treated as the training set (general methods)/source domain (transfer methods), and the PCQA datasets (e.g., SJTU-PCQA and WPC) are treated as the testing set (general methods)/the target domain and testing set (transfer methods). Note that the whole data and labels of the IQA dataset are used in training, and only the data itself of the PCQA dataset may be involved in feature alignment for transfer methods. The performance is evaluated following the training-testing setups mentioned above. The results are shown in Table \ref{tab:ItoPC}.

\begin{table}[htbp]
  \centering
  \caption{Image-to-point cloud transfer performance comparison. The general IQA and transfer methods are trained on the IQA dataset and tested on the PCQA dataset. Modal I means that the methods are designed to operate within a single image domain, modal I-to-I indicates that the methods operate on image-to-image transfer, and I-to-P represents the image-to-point cloud transfer. The proposed DWIT-PCQA exhibits the best transfer performance and significant improvement compared to the existing image-to-point cloud transfer method IT-PCQA. }
  \subtable[Performance of training on TID-2013 and testing on SJTU-PCQA.]{
    \begin{tabular}{l|c|cccc}
    \hline
          & modal & PLCC  & SROCC & KROCC & RMSE \\
    \hline
    DBCNN & I &0.405  & 0.399  & 0.310  & 1.370  \\
    HyperIQA & I & 0.371  & 0.361  & 0.277  & 1.625  \\
    StyleAM & I-to-I & 0.610  & 0.575  & 0.415  & 1.868  \\
    Chen's & I-to-I & 0.711  & 0.669  & 0.482  & 1.646  \\
    \hline
    IT-PCQA & I-to-PC & 0.717  & 0.678  & 0.494  & 1.636  \\
    DWIT-PCQA & I-to-PC & \textbf{0.859} & \textbf{0.828} & \textbf{0.640} & \textbf{1.197} \\
     Gain & I-to-PC & \textbf{$\uparrow$19.8\%} & \textbf{$\uparrow$22.1\%} & \textbf{$\uparrow$29.6\%} & \textbf{$\downarrow$26.8\%}\\
    \hline
    \end{tabular}%
    \label{tab:ItoPC1}%
  }
  \quad
  \subtable[Performance of training on TID2013 and testing on WPC.]{
    \begin{tabular}{l|c|cccc}
    \hline
        & modal  & PLCC  & SROCC & KROCC & RMSE \\
    \hline
    DBCNN & I & 0.268  & 0.280  & 0.202  & 25.266  \\
    HyperIQA & I & 0.209  & 0.229  & 0.163  & 25.539  \\
    StyleAM & I-to-I & 0.379  & 0.325  & 0.230  & 21.111  \\
    Chen's & I-to-I & 0.456  & 0.425  & 0.297  & 20.258 \\
    \hline
    IT-PCQA & I-to-PC & 0.439  & 0.437  & 0.265  & 20.470 \\
    DWIT-PCQA & I-to-PC & \textbf{0.628} & \textbf{0.617} & \textbf{0.454} & \textbf{17.783}  \\
    Gain & I-to-PC & \textbf{$\uparrow$43.1\%} & \textbf{$\uparrow$41.2\%} & \textbf{$\uparrow$71.3\%} & \textbf{$\downarrow$13.1\%} \\
    \hline
    \end{tabular}%
    \label{tab:ItoPC2}%
  }
  \quad
  \subtable[Performance of training on KADID-10k and testing on SJTU-PCQA.]{
    \begin{tabular}{l|c|cccc}
    \hline
        & modal  & PLCC  & SROCC & KROCC & RMSE \\
    \hline
    DBCNN & I & 0.371  & 0.368  & 0.272  & 2.566  \\
    HyperIQA & I & 0.349  & 0.369  & 0.283  & 2.609  \\
    StyleAM & I-to-I & 0.714  & 0.705  & 0.519  & 1.521  \\
    Chen's & I-to-I & 0.609  & 0.499  & 0.349  & 1.819 \\
    \hline
    IT-PCQA & I-to-PC & 0.724  & 0.660  & 0.466  & 1.628 \\
    DWIT-PCQA & I-to-PC & \textbf{0.827} & \textbf{0.812} & \textbf{0.625} & \textbf{1.284}  \\
    Gain & I-to-PC & \textbf{$\uparrow$14.2\%} & \textbf{$\uparrow$23.0\%} & \textbf{$\uparrow$34.1\%} & \textbf{$\downarrow$21.1\%} \\
    \hline
    \end{tabular}%
    \label{tab:ItoPC3}%
  }
  \quad
  \subtable[Performance of training on KADID-10k and testing on WPC.]{
    \begin{tabular}{l|c|cccc}
    \hline
        & modal  & PLCC  & SROCC & KROCC & RMSE \\
    \hline
    DBCNN & I & 0.259  & 0.272  & 0.195  & 26.571  \\
    HyperIQA & I & 0.221  & 0.240  & 0.172  & 26.641  \\
    StyleAM & I-to-I & 0.378  & 0.336  & 0.233  & 21.102  \\
    Chen's & I-to-I & 0.519  & 0.499  & 0.350  & 19.395 \\
    \hline
    IT-PCQA & I-to-PC & 0.560  & 0.539  & 0.383  & 18.758 \\
    DWIT-PCQA & I-to-PC & \textbf{0.703} & \textbf{0.714} & \textbf{0.537} & \textbf{15.971}  \\
    Gain & I-to-PC & \textbf{$\uparrow$25.5\%} & \textbf{$\uparrow$34.5\%} & \textbf{$\uparrow$40.2\%} & \textbf{$\downarrow$14.9\%} \\
    \hline
    \end{tabular}%
    \label{tab:ItoPC4}%
  }
  \label{tab:ItoPC}
\end{table}

We can see from Table \ref{tab:ItoPC} that: i) the proposed DWIT-PCQA exhibits the best transfer performance when generalized from IQA to PCQA datasets. DWIT-PCQA outperforms IT-PCQA~\cite{Yang2022ITPCQA} by about 20\% in terms of SROCC on SJTU-PCQA, and 40\% on WPC (trained with the same source domain), which demonstrates the effectiveness of the proposed alignment strategy; ii) directly transferring the IQA model from images to point clouds shows poor performance, indicating that the domain discrepancy between the source and target domains lead to the requirement for corresponding processing (i.e. feature alignment in transfer methods); iii) the existing transfer methods exhibit unsatisfactory performance. Among them, the style features in StyleAM are not generalized under huge domain discrepancies. The rudimentary alignment in Chen's method is insufficient to solve difficult image-to-point cloud transfer. The relaxation mechanism in IT-PCQA looses training when alignment violates feature-to-quality mapping, which reduces fitting ability.

Additionally, we present examples of distorted point clouds with their subjective MOS and predicted quality score generated by the proposed DWIT-PCQA in Fig. \ref{fig:examples}. The predicted quality scores align closely with subjective perceptions, despite being derived from unlabeled training. 

% for rudimentary feature alignment 

\subsection{Performance Comparison with General PCQA Methods which Require Labeled Training}
\label{sec:overallperformance}

In addition to comparing the proposed DWIT-PCQA with the existing image-to-point cloud transfer method IT-PCQA~\cite{Yang2022ITPCQA}, we also compare our DWIT-PCQA with prevalent general FR-PCQA metrics, including MSE-PSNR-P2point (M-p2po)~\cite{cignoni1998metro}, Hausdorff-PSNR-P2point (H-p2po)~\cite{cignoni1998metro}, MSE-PSNR-P2plane (M-p2pl)~\cite{Mekuria2016Evaluation}, Hausdorff-PSNR-P2plane (H-p2pl)~\cite{Mekuria2016Evaluation}, PSNRyuv~\cite{MPEGSoft}, PCQM~\cite{meynet2020pcmd}, GraphSIM~\cite{yang2020graphsim}, MPED~\cite{yang2021MPED} and pointSSIM~\cite{Alexiou2020pointSSIM}, and prevalent general NR-PCQA methods, including ResSCNN~\cite{Liu2022LSPCQA}, GPA-Net~\cite{shan2022GPANet}, and PQA-Net~\cite{Liu2021PQANet}. Note that these general NR-PCQA metrics need the quality labels of point clouds for training. For the transfer methods, we use TID2013 and KADID-10k as the source domain and treat SJTU-PCQA and WPC as the target domain, respectively. The results are shown in Table \ref{tab:overall}.%, and the best performance is highlighted in bold. 

% Table generated by Excel2LaTeX from sheet 'overall'
\begin{table*}[htbp]
  \centering
  \caption{Performance comparison with general PCQA methods. The general methods require labeled training, and the transfer methods are evaluated under unlabeled training. The proposed DWIT-PCQA presents competitive performance compared to conventional FR-PCQA and NR-PCQA metrics.}
    \begin{tabular}{c|l|c|cccc|cccc}
    \hline
    \multirow{2}{*}{\centering Type} & \multirow{2}{*}{\centering Method} & \multicolumn{1}{c|}{\multirow{2}{*}{\centering Requiring Labels}} & \multicolumn{4}{c|}{SJTU-PCQA} & \multicolumn{4}{c}{WPC} \\
\cline{4-11}          &       &       & PLCC  & SROCC & KROCC & RMSE  & PLCC  & SROCC & KROCC & RMSE \\
    \hline
    \multirow{9}[1]{*}{FR} & M-p2po & yes   & 0.889  & 0.803  & 0.623  & 1.093  & 0.586  & 0.567  & 0.414  & 18.452  \\
          & M-p2pl & yes   & 0.776  & 0.689  & 0.517  & 1.499  & 0.468  & 0.445  & 0.333  & 20.102  \\
          & H-p2po & yes   & 0.776  & 0.696  & 0.519  & 1.482  & 0.379  & 0.256  & 0.176  & 21.136  \\
          & H-p2pl & yes   & 0.752  & 0.686  & 0.506  & 1.566  & 0.369  & 0.313  & 0.219  & 21.231  \\
          & PSNRyuv & yes   & 0.743  & 0.737  & 0.550  & 1.581  & 0.579  & 0.563  & 0.401  & 18.597  \\
          & PCQM  & yes   & 0.795  & 0.869  & 0.681  & 2.383  & 0.705  & 0.750  & 0.571 & 22.891  \\
          & GraphSIM & yes   & 0.877  & 0.858  & 0.667  & 1.126  & 0.693  & 0.679  & 0.495  & 16.487  \\
          & MPED  & yes   & 0.907  & 0.890  & 0.717  & 0.997  & 0.692  & 0.682  & 0.512  & 16.426  \\
          & pointSSIM & yes   & 0.734  & 0.709  & 0.521  & 1.609  & 0.481  & 0.453  & 0.333  & 20.037  \\
          \hline
    \multirow{7}[3]{*}{NR} & ResSCNN  & yes   & 0.867  & 0.858  & 0.654  & 1.113  & 0.747  & 0.728  & 0.517  & 16.131  \\
          & GPA-Net & yes   & 0.864  & 0.855  & 0.632  & 1.058  & 0.722  & 0.704  & 0.497  & 15.784  \\
          & PQA-Net & yes   & 0.859  & 0.836  & 0.646  & 1.072  & 0.718  & 0.703  & 0.508  & 15.073  \\
\cline{2-11}          & IT-PCQA (TID2013) & no    & 0.717  & 0.678  & 0.494  & 1.636  & 0.439  & 0.437  & 0.265  & 20.470  \\
& IT-PCQA (KADID-10k) & no    & 0.724  & 0.660  & 0.466  & 1.628  & 0.560  & 0.539  & 0.383  & 18.758  \\
          & DWIT-PCQA (TID2013) & no    & 0.859 & 0.828 & 0.640 & 1.197 & 0.628 & 0.617 & 0.454 & 17.783 \\
          & DWIT-PCQA (KADID-10k) & no    & 0.827 & 0.812 & 0.625 & 1.284 & 0.703 & 0.714 & 0.537 & 15.971 \\
    \hline
    \end{tabular}%
  \label{tab:overall}%
\end{table*}%

% & MM-PCQA & yes   & 0.924 & 0.902 & 0.738 & 0.882 & 0.840 & 0.825 & 0.497  & 11.643 \\

We can see from Table \ref{tab:overall} that: i) compared to the best general metrics, the performance degradation of the proposed DWIT-PCQA is only about 7.0\% in terms of SROCC on SJTU-PCQA, and 4.8\% on WPC, which indicates that the proposed transfer metrics have been of great practical value; ii) although the quality of point clouds is not used for training, the proposed DWIT-PCQA still exhibits better performance than many general FR-PCQA and NR-PCQA metrics on the PCQA datasets, which indicates the benefits of prior knowledge in IQA for PCQA, and also highlights the importance of how to bridge the two media.

\subsection{Ablation Study}\label{sec:abla}

In this part, we illustrate the effectiveness of each loss function. Specifically, TID2013 is used as the source domain, and SJTU-PCQA is used as the target domain. Considering the interdependence between different modules, we adopted a progressive ablation study approach. We use $\Lset_{Reg}$ to train the network as a benchmark, which demonstrates the performance of directly predicting point cloud quality with training in images. In order to demonstrate the effect of direct feature alignment as in \cite{Yang2022ITPCQA}, we use $\Lset_{Reg}+\Lset_{DWCE}$ (w/o $\w_{y}$) as the loss function to repeat the trials where $\w_{y}$ in $\Lset_{DWCE}$ is set to 1. Then to evaluate the effect of the distortion-guided distribution alignment, $\Lset_{Reg}+\Lset_{DWCE}+\Lset_{Cls}$ (w/ $\w_{y}$) is adopted as the loss function to train the network where the calculation of $\w_{y}$ depends on $\Lset_{Cls}$. To demonstrate the effect of quality-aware feature disentanglement which is used to alleviate the disruption of feature-to-quality mapping caused by distortion-based feature alignment mentioned earlier, $\Lset_{Reg}+\Lset_{DWCE}+\Lset_{Cls}+\Lset_{Fea}$ generates the final performance of the network. Finally, to illustrate the impact of multitasking branches on the DA framework and the further role of quality-aware feature disentanglement, the performance between ${{\mathcal{L}}_{Reg}}+{{\mathcal{L}}_{Cls}}$ and ${{\mathcal{L}}_{Reg}}+{{\mathcal{L}}_{Cls}}+{{\mathcal{L}}_{Fea}}$ is compared. The results are shown in Table \ref{tab:ab_loss}.

% Table generated by Excel2LaTeX from sheet 'ablation study'
\begin{table}[htbp]
\setlength\tabcolsep{3pt}
  \centering
  \caption{Model performance with different loss functions on TID2013 (source domain) and SJTU-PCQA (target domain).}
    \begin{tabular}{cccc|cccc}
    \hline
    \multicolumn{1}{l}{$\mathcal{L}_{Reg}$} & \multicolumn{1}{l}{$\mathcal{L}_{DWCE}$} & \multicolumn{1}{l}{$\mathcal{L}_{Cls}$} & \multicolumn{1}{l|}{$\mathcal{L}_{Fea}$} & \multicolumn{1}{l}{PLCC} & \multicolumn{1}{l}{SROCC} & \multicolumn{1}{l}{KROCC}& \multicolumn{1}{l}{RMSE}\\
    \hline
    \checkmark     & \XSolidBrush     & \XSolidBrush     & \XSolidBrush     & 0.684  & 0.667 & 0.495 & 1.679 \\
    \checkmark     & \checkmark (w/o $\w_{y}$)    & \XSolidBrush & \XSolidBrush & 0.711 & 0.704 & 0.517 & 1.581 \\
    \checkmark     & \checkmark (w/ $\w_{y}$)    & \checkmark     & \XSolidBrush &  0.778 &  0.749 & 0.548 & 1.458 \\
    \checkmark     & \checkmark     & \checkmark     & \checkmark     & \textbf{0.859} & \textbf{0.828} & \textbf{0.640} & \textbf{1.197}\\
    \hline
    \checkmark     & \XSolidBrush     & \checkmark     & \XSolidBrush     & 0.608  & 0.579 & 0.406 & 1.936 \\
    \checkmark     & \XSolidBrush     & \checkmark     & \checkmark     & 0.664  & 0.648 & 0.482 & 1.715 \\
    \hline
    \end{tabular}%
  \label{tab:ab_loss}%
\end{table}%

\begin{figure}[t]
\centering
  \subfigure[Irregular feature space of direct feature alignment.]{\includegraphics[width=1\linewidth]{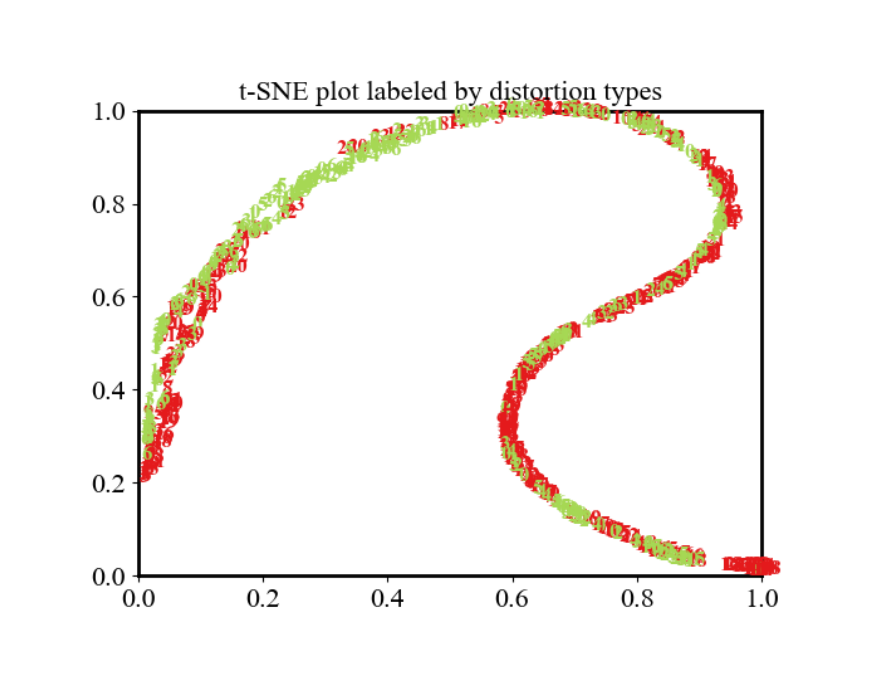}\label{sfig:tsne_DA_only}}
  \subfigure[Regular feature space of proposed distortion-based conditional feature alignment.]{\includegraphics[width=1\linewidth]{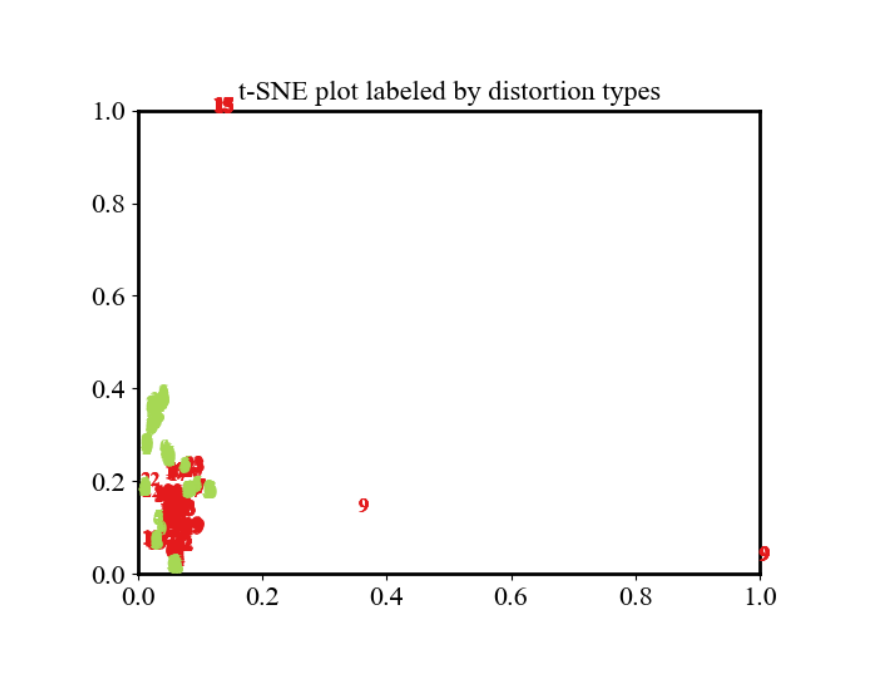}\label{sfig:tsne_DA_w}}
  \centering
\caption{T-SNE plot under different alignment conditions. Following similar setups in Fig. \ref{fig:tsne}, features of the direct alignment and proposed conditional alignment are extracted from the source domain (red) and the target domain (green). The involvement of unique distortions in direct alignment results in a specialized and irregular feature space which is not conducive to generalization. Instead, the proposed conditional alignment emphasizes the cross-domain common distortion patterns, which is beneficial for the model's generalization.}
  \label{fig:tsneComDA}
\end{figure}

We can see from Table \ref{tab:ab_loss} that: i) according to the performance of $\Lset_{Reg}+\Lset_{DWCE}$ (w/o $\w_{y}$),  the performance gain demonstrates that transferring the prior knowledge from images to point clouds for quality assessment task is feasible; ii) based on the performance of $\Lset_{Reg}+\Lset_{DWCE}+\Lset_{Cls}$ (w/ $\w_{y}$), the distortion-based importance-weighted alignment is necessary and effective to handle the domain discrepancy induced by distortion differences between two domains; iii) the proposed quality-aware feature alignment can maintain the quality-sensitive representation during alignment with biased distortions, which is also the necessary optimization objective derived from the formula derivation to promote further alignment; iv) the performance degradation of ${{\mathcal{L}}_{Reg}}+{{\mathcal{L}}_{Cls}}$ further illustrates that the different requirements for features of different task branches will suppress the feature alignment in DA. Considering the necessity of multitasking branches in the proposed DA framework, the proposed quality-aware feature disentanglement can alleviate the conflict of feature distribution between different task branches.

% i) only $\Lset_{Reg}$ generates the basic performance. Note that the performance here surpasses IQA methods and IT-PCQA in Table \ref{tab:ItoPC} because the adopted ResNet-50 network has been initialized with the public pre-trained model, which proves that generalized features derived from large-scale training are beneficial for cross-domain transfer;

\subsection{Feature Space Analysis}

To further demonstrate the effectiveness of the proposed method, the t-SNE plots of extracted features under forced direct alignment and proposed conditional alignment are given in Fig. \ref{fig:tsneComDA}. Fig. \ref{sfig:tsne_DA_only} shows the t-SNE plot derived from direct alignment, where the features of two domains are aligned into a specialized and irregular feature space. The direct alignment forcefully fits the unique distortions and reduces potential generalization, posing challenges to the feature-to-quality mapping. Instead, the proposed conditional alignment emphasizes the cross-domain common distortion patterns, forming a regular feature space that is more conducive to generalizing the feature-to-quality mapping as shown in Fig. \ref{sfig:tsne_DA_w}.

\section{Conclusion}\label{sec:conclusion}

In this paper, we proposed a novel transfer quality assessment method called DWIT-PCQA for point clouds by leveraging the prior knowledge of images, which solves cross-media transfer quality assessment by decomposing the optimization objective of the conventional DA into two suboptimization functions with distortion as a transition. Through network implementation, the proposed method incorporates distortion-based feature distribution re-weighting into the unsupervised adversarial DA framework to achieve more reasonable alignment under distortion discrepancy. Furthermore, the proposed quality-aware feature disentanglement is implemented to maintain the sensitivity of features to perceptual quality during alignment with biased distortions. The proposed method further reveals the potential connection between different types of media in the field of quality assessment, and exhibits outstanding and reliable performance.

\bibliographystyle{IEEEtran}
\bibliography{manuscript}

%%%%%%%%%%%%%%%%%%%%%%%%%%%%%%%%%%%%%%%%%%%%%%%%%%%%%%%%%%%%

\end{document}